\documentclass[12pt, a4paper]{article}
\usepackage{amsmath,amsfonts,amsthm,mathrsfs}
\usepackage[page,title,titletoc,header]{appendix}
\usepackage{color}

\numberwithin{equation}{section}

\newtheorem{thm}{Theorem}

\newtheorem{lem}{Lemma}
\newtheorem{prop}{Proposition}
\newtheorem{remark}{Remark}
\newtheorem{definition}{Definition}

\newcommand{\thmref}[1]{Theorem~\ref{#1}}

\newcommand{\lemref}[1]{Lemma~\ref{#1}}

\newcommand{\propref}[1]{Proposition~\ref{#1}}

\def\sph{\mathbb{S}^{d-1}}
\def\RR{\mathbb{R}}

\def\a{\alpha}
\def\s{\sigma}
\def\NN{\mathbb{N}}
\def\ZZ{\mathbb{Z}}
\def\f{\frac}
\def\d{\delta}
\def\T{\mathbf \Theta}

\def\OO{\mathcal{O}}
\def\l{\lambda}
\def\la{\langle}
\def\ra{\rangle}
\def\CH{\mathcal H}
\def\proj{\operatorname{Proj}}

\def\Bl{\Bigl} \def\Br{\Bigr}

\begin{document}
%\title{Spherical approximation by deep convolutional neuron networks}
\title{Theory of Deep Convolutional Neural Networks II: Spherical Analysis}
\author{Zhiying Fang$^{a}$, Han Feng$^{b}$, Shuo Huang$^{b}$, Ding-Xuan Zhou$^{a, b}$ \\
 School of Data Science$^{a}$,  Department of Mathematics$^{b}$ \\
City University of Hong Kong, Kowloon, Hong Kong \\
Email: zyfang4-c@my.cityu.edu.hk, hanfeng@cityu.edu.hk \\
shuang56-c@my.cityu.edu.hk, mazhou@cityu.edu.hk}
\date{}

\maketitle
\begin{abstract}
Deep learning based on deep neural networks of various structures and architectures has been powerful in many practical applications,
but it lacks enough theoretical verifications. In this paper, we consider a family of deep convolutional neural networks
applied to approximate functions on the unit sphere $\mathbb{S}^{d-1}$ of $\mathbb{R}^d$. Our analysis presents rates of uniform approximation when the approximated function
lies in the Sobolev space $W^r_\infty (\sph)$ with $r>0$ or takes an additive ridge form. Our work verifies theoretically the modelling and approximation ability
of deep convolutional neural networks followed by downsampling and one fully connected layer or two.
The key idea of our spherical analysis is to use the inner product form of the reproducing kernels of the spaces of spherical harmonics
and then to apply convolutional factorizations of filters to realize the generated linear features.
\end{abstract}


\section{Introduction}

{\bf Deep learning} has attracted tremendous attention from various fields of science and technology recently.
Wide applications including those in image processing \cite{he2016deep} and speech recognition \cite{hinton2012deep} have
received great successes. Based on deep neural network structures, it has a strong capability of obtaining data
features and distinguishes itself from classical machine learning methods. Though it is successful in practical applications, theoretical assurances
are still lacking and need to be investigated. Many attempts have been made trying to understand the practical power of
deep neural networks \cite{bruna2013invariant,mallat2016understanding}.

The classical (shallow) neural networks to approximate functions or process data
on $\RR^d$ take the form
\begin{equation}\label{ShallowNetworks}
f_{N}(x)=\sum_{k=1}^{N} c_{k} \sigma\left(\langle w_k,x\rangle-b_{k}\right), \qquad x\in\RR^d,
\end{equation}
where $N$ is the number of neurons called width, $\{w_k \in \RR^d, c_k\in\RR, b_k\in\RR\}_{k=1}^N$ are parameters corresponding to connection weights, biases, and coefficients,
$\langle w_k, x\rangle =w_k \cdot x$ is the inner product in $\RR^d$, and $\s:\RR\to \RR$ is an activation function. Approximation of functions on subsets
of $\RR^d$ by shallow neural networks (\ref{ShallowNetworks}) was studied well around the late 1980s. See \cite{Mhaskar1993, Chui1996} and references therein.
As a natural extension of shallow nets, fully connected deep neural networks (DNNs) have been developed since the 1990s.
A fully connected DNN of input $x=(x_1,x_2,\ldots, x_d)\in \RR^d$ with $J$ hidden layers of neurons $\{H^{(j)}: \RR^d\to \RR^{d_j}\}$ with width $\{d_j\}$ is defined iteratively by $H^{(0)}(x)=x$ with $d_0 =d$ and
$$ \left(H^{(j)}(x)\right)_i=\s(\langle w_i^{(j)}, H^{(j-1)}(x)\rangle-b_i^{(j)}), \quad i=1,2,\ldots, d_j, $$
where  $w_i^{(j)}\in \RR^{d_{j-1}}$ and $b^{(j)}_i\in\RR$ are connection weights and biases in the $j$-th layer.
If we use $w_i^{(j)}\in \RR^{d_{j-1}}$ with $i=1, \ldots, d_j$ as rows to form a $d_{j}\times d_{j-1}$ matrix $F^{(j)}$
and $b_i^{(j)}$ to form a vector $b^{(j)} =(b_i^{(j)})_{i=1}^{d_j}$, then by acting $\s$ componentwise on vectors, the DNN of depth $J$ can be expressed as
\begin{equation}\label{dcnn-1}
H^{(j)}(x)=\s \left(F^{(j)} \ H^{(j-1)}(x)- b^{(j)}\right), \quad j=1,2, \ldots, J.
\end{equation}

DNNs designed by convolutions are called {\bf deep convolutional neural networks} (CNNs) and have been very successful
in image classification and related applications \cite{Krizhevsky2012}. Such a network is associated with
a sequence of convolutional filters $\mathbf w=\{w^{(j)}: \ZZ \to \RR\}_{j=1}^J$, where $w^{(j)}$ is supported in $\left\{0, \cdots, S^{(j)}\right\}$ for some $S^{(j)} \in \mathbb{N}$ called filter length. The convolution of such a filter $w$ supported in $\left\{0,\cdots, S\right\}$ with another sequence $v=(v_1, \ldots, v_D)$ is a sequence $w*v$ given by
$$(w*v)_i=\sum_{k\in\ZZ} w_{i-k} v_k =\sum_{k=1}^D w_{i-k}v_k, \qquad i\in\ZZ,$$
which is supported in $\left\{1,\cdots, D+S\right\}$. Then by restricting the convoluted sequence onto its support, for input $x=(x_1,x_2,\ldots, x_d)\in \RR^d$, a deep CNNs with $J$ hidden layers of neurons $\{h^{(j)}: \RR^d\to \RR^{d_j}\}$ and widths $\left\{d_j = d_{j-1} + S^{(j)}\right\}$ is defined iteratively by $h^{(0)}(x)=x$ and
\begin{equation}\label{dcnn-2}
  h^{(j)}(x)=\s\left(\left(\sum_{k=1}^{d_{j-1}} w^{(j)}_{i-k} \left(h^{(j-1)}(x)\right)_k\right)_{i=1}^{d_j} -b^{(j)}\right).
\end{equation}
By inducing a Toeplitz type {\bf convolutional matrix}
$$T^{w}:=\left(w_{i-k}\right)_{i=1, \dots, D+S, k=1, \dots, D}$$
associated with a filter $w$ of filter length $S$ and $D\in\NN$ given explicitly by
\begin{equation}\label{Toeplitz}
T^{w}=\left[\begin{array}{lllllll}
{w_{0}} & {0} & {0} & {0} & {\cdots} & {\cdots} & {0} \\
{w_{1}} & {w_{0}} & {0} & {0} & {\cdots} & {\cdots} & {0} \\
{\vdots} & {\ddots} & {\ddots} & {\ddots} & {\ddots} & {\ddots} & {\vdots} \\
{w_{S}} & {w_{S-1}} & {\cdots} & {w_{0}} & {0} & {\cdots} & {0} \\
{0} & {w_{S}} & {\cdots} & {w_{1}} & {w_{0}} & {0 \cdots} & {0} \\
{\vdots} & {\ddots} & {\ddots} & {\ddots} & {\ddots} & {\ddots} & {\vdots} \\
{0} & {\cdots} & {0} & {w_{S}} & {\cdots} & {w_{1}} & {w_{0}} \\
{0} & {\cdots} & {0} & {0} & {w_{S}} & {\cdots} & {w_{1}} \\
{\vdots} & {\ddots} & {\ddots} & {\ddots} & {\ddots} & {\ddots} & {\vdots} \\
{0} & {0} & {0} & {0} & {\cdots} & {0} & {w_{S}}
\end{array}\right]\in \mathbb{R}^{(D+S) \times D},
\end{equation}
a deep CNN can be regarded as a special sparse form of a fully connected DNNs with the full connection matrices $F^{(j)}$
in \eqref{dcnn-1} replaced by the sparse Toeplitz convolutional matrices $T^{(j)}:=T^{w^{(j)}}$ with $D=d_{j-1}$ and $S=S^{(j)}$ for $j=0,1,\ldots, J$.
That is, \eqref{dcnn-2} becomes
\begin{equation}\label{InitialLayers}
  h^{(j)}(x)=\s \left(T^{(j)}\  h^{(j-1)}(x)-b^{(j)}\right), \qquad j=1,\ldots, J.
\end{equation}
Compared with fully-connected DNNs, deep CNNs reduce the computational complexity by using at each layer a Toeplitz matrix due to the sparsity and convolutional nature.
Throughout the paper we take an identical filter length $S^{(j)} \equiv S\in\NN$ implying $\{d_j = d+ j S\}$ and take the rectified linear unit (ReLU) activation function
$$\s(u)=\max\{u, 0\}, \qquad u\in\RR. $$
It was shown in \cite{zhou2020theory} that the output layer of any fully-connected DNN can be realized
by a downsampled deep CNN with free parameters of the same order, and that deep CNNs can approximate
ridge functions of the form $g(\xi\cdot x)$ with univariate functions $g$ and unknown $\xi\in \RR^d$ to the same accuracy with much smaller number of free parameters.
Universality of approximation by deep CNNs was also established in \cite{zhou2020universality} where rates of approximation were provided for restrictions of functions
from the Sobolev space $H^r(\RR^d)$ with a regularity index $r>2+d/2$. The regularity index $r$ is large when $d$ increases and is essentially
needed in the analysis there due to the function regularity on the whole Euclidean space $\RR^d$.
Note that the issue of approximating non-smooth functions by fully connected DNNs has been studied in \cite{Suzuki2019, Imaizumi2019}.
Moreover, the Sobolev space $H^r(\RR^d)$ requires derivatives of various orders to belong to the $L_2$ space,
while the approximation considered in \cite{zhou2020universality} is measured in the $L_\infty$ norm.

In this paper, we overcome the difficulty in the large regularity index and the inconsistency of $L_2$ and $L_\infty$ norms
for the setting with data from the unit sphere $\sph$ in $\RR^d$.
With our novel analysis conducted with spherical harmonic expansions, we can present rates of approximating functions
from the Sobolev space $W^r_\infty(\sph)$ (to be defined below) on $\sph$, with any positive index $r$, by downsampled deep CNNs defined in \cite{zhou2020theory} followed by two fully connected layers.

\begin{definition}
The {\bf downsampling} operator $\mathfrak{D}_d: \mathbb{R}^{D} \rightarrow \mathbb{R}^{\lfloor D/d\rfloor}$ with a scaling parameter $d \leq D$ is defined by
\begin{equation}\label{downsample}
\mathfrak{D}_d (v) = \left(v_{id}\right)_{i=1}^{\lfloor D/d \rfloor}, \qquad v=(v_i)_{i=1}^D \in \mathbb{R}^D,
\end{equation}
where $\lfloor u\rfloor$ denotes the integer part of $u> 0$.
\end{definition}

After the last CNN layer, we add two fully connected layers $h^{(J+1)}, h^{(J+2)}$  with widths $\mathcal D_1, \mathcal D_2>0$, respectively, connection matrices $F^{(J+1)}, F^{(J+2)}$ and  bias vectors $b^{(J+1)}, b^{(J+2)}$, to be determined. Precisely,
\begin{equation}\label{LastLayer}
 h^{(J+1)}(x)=\s\left(F^{(J+1)} \mathfrak{D}_d \left(h^{(J)}(x)\right)-b^{(J+1)}\right)
 \end{equation}
and
 \begin{equation}
 h^{(J+2)}(x)=\s\left(F^{(J+2)} h^{(J+1)}(x)-b^{(J+2)}\right).
\end{equation}
Such a network with many convolutional layers followed by downsampling operations and very few fully connected layers
is quite common in practical applications \cite{Krizhevsky2012, he2016deep}. The hypothesis space of functions on $\mathbb{S}^{d-1}$ induced by our network is given by
\begin{equation}\label{hypothesisspace12}
\mathfrak{H}_{J,\mathcal{D}_1,\mathcal D_2,S}=%^{\textbf{w},\textbf{b},\textbf{S},\mathcal{J}} =
\left\{c^{(J+2)} \cdot h^{(J+2)}(x)-A: \ c^{(J+2)} \in \mathbb{R}^{{\mathcal D}_2}, A\in\RR\right\}.
\end{equation}

\section{Main Results on Rates of Approximation}\label{mainresult}

Our target is to establish rates of approximating functions in $W^r_\infty(\sph)$ by those from the hypothesis space
$\mathfrak{H}_{J,\mathcal{D}_1,\mathcal D_2,S}$ defined by (\ref{hypothesisspace12}).
Since the sums of the rows in the middle of the Toeplitz type matrix (\ref{Toeplitz}) are equal, we impose
for the bias vectors $\{b^{(j)}\}_{j=1}^J$ of the convolutional layers a restriction
\begin{equation}\label{restrrow}
b^{(j)}_{S+1} = \ldots = b^{(j)}_{d_j -S}, \qquad j=1, \ldots, J.
\end{equation}
For the two fully connected layers we take the widths
\begin{equation}\label{fullwidths}
\mathcal{D}_1=(2N+3)\lfloor(d+JS)/d\rfloor, \qquad \mathcal{D}_2=\lfloor(d+JS)/d\rfloor
\end{equation}
for some positive integer $N\in\NN$ and connection matrices
\begin{equation}\label{fullmatrices}
F^{(J+1)} =\Xi_{\mathcal{D}_2, \textbf{1}_{2N+3}}, \qquad F^{(J+2)} = \Xi_{\mathcal{D}_2, \Theta_N}^T
\end{equation}
with $\textbf{1}_{2N+3}=(1,1,\ldots,1)^T\in \RR^{2N+3}$ and $\Theta_N =(\theta_1, \ldots, \theta_{2N+3})^T \in \RR^{2N+3}$.
Here the matrix $\Xi_{\mathcal{D}_2, \vec u}$ takes a block form as
$$ \Xi_{\mathcal{D}_2, \vec u} = \left[\begin{array}{cccc}
\vec u & O & O \cdots & O \\
O & \vec u & O \cdots & O \\
\vdots & \ddots & \ddots & \vdots \\
O & \cdots  & O & \vec u
\end{array}\right] \in\RR^{(2N+3)\mathcal{D}_2 \times \mathcal{D}_2}, \ \vec u =\left[\begin{array}{l}
u_1 \\
\vdots \\
u_{2N+3}
\end{array}\right] \in \RR^{2N+3}. $$

Our first main result, to be proved in Section \ref{proof of main results}, can be stated as follows.
A positive parameter $\tau$ (which can be arbitrarily small) is needed due to the continuous embedding
of the Sobolev space $W^s_2(\sph)$ with $s>\frac{d-1}{2}$ into $C(\sph)$
found in Proposition \ref{embedding} below.

\begin{thm}\label{theorem:mainresult1}
Let $2 \leq S \leq d$, $d \geq 3$, $J \geq \frac{d-1}{S-1}$, $0< r\not= d-1$, and $\tau >0$ satisfy $\tau< r-(d-1)$ when $r>d-1$. Take
$$ N=\left\{\begin{array}{ll}
\left\lfloor \lfloor \frac{(S-1)J +1}{d}\rfloor^{\frac{1}{2(d-1+\tau)}} \right\rfloor^{d+1}, & \hbox{if} \ r < d-1, \\
\left\lfloor \left\lfloor \lfloor \frac{(S-1)J +1}{d}\rfloor^{\frac{1}{2r}} \right\rfloor^{2+r} \right\rfloor, & \hbox{if} \ r > d-1.
\end{array}\right.$$
Then for any $f\in W^r_\infty(\sph)$, there exists a deep neural network consisting of $J$ layers of CNNs with filters of length $S$ and bias vectors
satisfying (\ref{restrrow}) followed by downsampling and two fully connected layers with widths (\ref{fullwidths}) and connection matrices (\ref{fullmatrices})
such that the hypothesis space $\mathfrak{H}_{J,\mathcal{D}_1, \mathcal D_2,S}$ contains a function $\hat f$ satisfying
$$
\left\|f-  \hat f\right\|_{\infty}
\leq C_{r, d, \tau, S}  J^{-\min\{\frac{r}{2(d-1+ \tau)}, \f 1 2\}} \|f\|_{W^r_\infty(\sph)},
$$
where $C_{r, d, \tau, S}$ is a constant independent of $J$ or $f$. Moreover,
the total number of free parameters $\mathcal{N}$ in the network can be bounded as
$$\mathcal{N} \leq \left(3 S+5\right) J +4.$$%^{1+\frac{r +d+1}{2(r+d-1+\tau)}}.$$
\end{thm}

\begin{remark}\label{accuracyThm1}
To achieve the approximation accuracy $\left\|f-  \hat f\right\|_{\infty} \leq \epsilon$, we only need to take
$$ J=\left\lceil\max\left\{\left(C_{r, d, \tau, S}\|f\|_{W^r_\infty(\sph)}/\epsilon\right)^{\max\left\{\frac{2(d-1+\tau)}{r}, 2\right\}},
\frac{d-1}{S-1}\right\}\right\rceil, $$
where we denote the  smallest integer greater than or equal to $u>0$ as $\lceil u \rceil$.
When $0< r <d-1$, we may set $\tau=1$ and see that the network achieving the approximation accuracy $\epsilon$ has depth
$J=\OO \left(\epsilon^{-\frac{2d}{r}}\right)$ and $\mathcal{N}=\OO \left(\epsilon^{-\frac{2d}{r}}\right)$
%$\OO \left(\epsilon^{-\frac{3(r+d)+1}{r}}\right)$
free parameters. When $r>d-1$, we may set $0<\tau < r-(d-1)$ and find the depth and the number of free parameters of this network
are both of orders $\OO \left(\epsilon^{-2}\right)$, slightly better than that in \cite{zhou2020universality}.
\end{remark}

Our second main result aims at  further demonstrating the superiority of deep CNNs over fully connected networks observed empirically in many practical applications.
Motivated by our earlier work \cite{zhou2020theory} on approximating ridge functions of type $g(y \cdot x)$ with $y\in\RR^d$, $g: \RR \to \RR$,
and additive models (see, e.g., \cite{ChristmanZhou, Schmidt-Hieber}) in statistics of the form $f(x) = \sum_{j=1}^d g_j (x_j)$ with univariate functions $\{g_j\}_{j=1}^d$,
we consider a family of {\bf additive ridge functions} of the form
\begin{equation}\label{additiveridge}
f(x) =\sum_{j=1}^m g_j (y_j \cdot x)
\end{equation}
with $y_j \in \sph, g_j: \RR \to \RR$ for $j\in\{1, \ldots, m\}$.
The following theorem to be proved in Section \ref{proof of main results} is about approximating additive ridge functions by deep CNNs followed
by downsampling and one fully-connected layer. For $0< \alpha \leq 1$, denote the space of Lipschitz-$\a$ functions on $[-1, 1]$
as $W^\alpha_\infty ([-1, 1])$ with the semi-norm $|\cdot|_{W^\alpha_\infty}$ being the Lipschitz-$\a$ constant.

\begin{thm}\label{thm2}
Let $m\in\NN$, $d\geq 3$, $2 \leq S \leq d$, $J= \left\lceil \frac{md-1}{S-1}\right\rceil$,
and $N \in \NN$. If $f$ is an additive ridge function (\ref{additiveridge}) with unknown $\{y_1, \ldots, y_m\} \subset \sph$,
$\{g_1, \ldots, g_m\}\subset W^\alpha_\infty ([-1, 1])$ for some $0<\a\leq 1$, then
there exists a deep neural network consisting of $J$ layers of CNNs with filters of length $S$ and bias vectors
 satisfying (\ref{restrrow}) followed by downsampling and one fully connected layer $h^{(J+1)}$ with width $(2N+3)\lfloor(d+JS)/d\rfloor$ and
 connection matrix $F^{(J+1)} =\Xi_{\lfloor(d+JS)/d\rfloor, \textbf{1}_{2N+3}}$
such that for some coefficient vector $c^{(J+1)}\in \RR^{(2N+3)\lfloor(d+JS)/d\rfloor}$ there holds
$$\left\|f - c^{(J+1)}\cdot h^{(J+1)}\right\|_\infty
\leq \sum_{j=1}^m |g_j|_{W^\alpha_\infty} N^{-\alpha}. $$
The total number of free parameters $\mathcal{N}$ in the network can be bounded as
$$
\mathcal{N} \leq (3S+2)\left\lceil \frac{md-1}{S-1}\right\rceil +m(2N+2). $$
\end{thm}

\begin{remark}\label{accuracyThm2}
To achieve an accuracy $\epsilon>0$ for approximating an additive ridge function (\ref{additiveridge}) on $\sph$, we only need to take
$$ N=\left\lceil \left(\sum_{j=1}^m |g_j|_{W^\alpha_\infty}\right)^{1/\alpha} \epsilon^{-\frac{1}{\alpha}}  \right\rceil = \OO \left(\epsilon^{-\frac{1}{\alpha}}\right). $$
The total number of free parameters of the achieving network is of orders $\OO \left(\epsilon^{-\frac{1}{\alpha}}\right)$.
This is the complexity required by the classical literature on fully-connected networks to achieve an accuracy $\epsilon>0$ for approximating univariate functions.
This extends our earlier work \cite{zhou2020theory} from the ridge case with $m=1$ to an additive ridge case with $m\in\NN$ and hints the superiority of deep CNNs in approximating multivariate functions of special structures.
\end{remark}

\section{Spherical Analysis for Deep CNNs}

In this section, we introduce ideas of our analysis before proving our main results.
We first give a brief review on relevant concepts from spherical harmonic analysis and introduce some classes of functions. More details can be found in \cite{dai2013approximation,wang2017fully}.

\subsection{Spherical harmonics and Sobolev spaces on spheres}

For $1\leq p\leq \infty$, we denote by $L_p(\mathbb{S}^{d-1})=L_p(\mathbb{S}^{d-1},\s_d)$ the $L_p$-function space defined with respect to the normalized Lebesgue measure $\s_d$ on $\mathbb{S}^{d-1}$, and $\|\cdot\|_p$ the norm of $L_p(\sph)$.

A spherical harmonic of degree $n\in\ZZ_+$ on $\sph$ is the restriction to $\sph$ of a homogeneous and harmonic polynomial of total degree $n$ defined on $\RR^{d}$.
Let $\mathcal{H}_n^d$ denote the set of all spherical harmonics of degree $n$ on $\sph$. It can be found in \cite{dai2013approximation} the dimension of the linear space $\mathcal{H}_n^d$ is
\begin{align}\label{H_dimension}
N(n,d) =   {{n+d-1}  \choose {n}}-{{n+d-3}\choose {n-2}} \leq C_d\, n^{d-2}, \quad \text{for}\ n\in\NN,
\end{align}
where $C_d>0$  only depends on $d$.
Note that $L_2(\sph)$ is a Hilbert space with inner product $\la f, g\ra_2:=\int_{\sph} f(x) g(x) d \sigma_{d}(x)$ for $f, g \in L_{2}\left(\sph\right)$. The spaces $\CH_n^d$, $n\in\ZZ_+$, of spherical harmonics are mutually orthogonal with respect to the inner product of $L_2(\sph)$. Since the space of spherical polynomials is dense in $L_2(\sph)$,  each $f\in L_2(\sph)$ has a spherical harmonic expansion:
\begin{equation*}
f=\sum_{n=0}^\infty \proj_n f = \sum_{n=0}^{\infty}\sum_{l=1}^{N(n,d)}\widehat{f}_{n,l}Y_{n,l}
\end{equation*}
converging in the $L_2(\sph)$ norm. Here and elsewhere, $\left\{Y_{n,l}\right\}_{l=1}^{N(n,d)}$ is an orthonormal basis of $\CH^d_n$, $\widehat{f}_{n,l}$ is the Fourier coefficients of $f$ given by \begin{align}
    \widehat{f}_{n,l}:=\left\langle f,Y_{n,l}\right\rangle_{L_2(\mathbb{S}^{d-1})}:= \int_{\mathbb{S}^{d-1}}f(x)Y_{n,l}(x)d\s_d(x),
\end{align}
and  $\proj_n$ is the orthogonal projection of $L_2(\sph)$ onto the subspace $\CH_k^d$ of spherical harmonics, which has an integral representation:
\begin{equation*}
    \proj_nf(x) =\int_{\sph} f(y) Z_n(x,y)\, d\s_d(y), \qquad  x\in\sph,
\end{equation*}
where
\begin{align*}
Z_n(x,y) = \sum_{i=1}^{N(n,d)} Y_{n,i}(x) Y_{n,i}(y), \qquad x,y \in \mathbb{S}^{d-1}.
\end{align*}
It can be readily shown that $Z_n(x,y)$ is the reproducing kernel of $\CH_n^d$ independent of the choice of  $\left\{Y_{n,l}\right\}_{l=1}^{N(n,d)}$. Furthermore,  with $\lambda = \frac{d-2}{2}$,
\begin{equation}\label{rkhsform}
Z_n(x,y) = \frac{n+\lambda}{\lambda} C^{\lambda}_n \left( \langle x,y \rangle\right), \qquad x,y\in\sph,
\end{equation}
where $C_n^\lambda(t)$ is the Gegenbauer polynomial of degree $n$ with parameter $\l>-1/2$, see, for instance, \cite{dai2013approximation}.

The spaces $\mathcal{H}_n^d$  of spherical harmonics can also be characterized as
eigenfunction spaces of the Laplace-Beltrami operator $\Delta_0$ on
$\sph$. Indeed,
\begin{align*}
    \mathcal{H}_n^d=\{f \in C^2(\mathbb{S}^{d-1}):\Delta_0 f=-\l_n f\},
\end{align*}
where $\lambda_n=n(n+d-2)$ and $C^2(\mathbb{S}^{d-1})$ denotes the space of all twice continuously differentiable functions on $\mathbb{S}^{d-1}$. As a matter of fact, we may define the fractional power $(-\Delta_0+I)^\alpha$ of $-\Delta_0+I$ for
$\alpha\in\RR$ in a distributional sense by
\[\proj_n((-\Delta_0+I)^\alpha f)=(\l_n+1)^{\alpha}\proj_n f,\]
which is a self-adjoint operator on $L_2(\sph)$ in the sense that
\[\Bl\la(-\Delta_0+I)^\alpha f, g\Br\ra_{L_2(\sph)}=\Bl\la f , (-\Delta_0+I)^\alpha g\Br\ra_{L_2(\sph)}, \quad \forall f,g\in L_2(\sph).\]

Now we define the Sobolev space $W_p^r(\mathbb{S}^{d-1})$ to be a subspace of $L_p(\mathbb{S}^{d-1})$, $1\leq p\leq \infty$, $r>0$, with the finite norm
\begin{align}\label{Sobolev_norm}
    \| f \|_{W_p^r(\mathbb{S}^{d-1})}:=&\left\| (-\Delta_0+I)^{r/2} f\right\|_p\\
    =&\left\|\sum_{n=0}^{\infty}(1+\lambda_n)^{\frac{r}{2}} \sum_{l=1}^{N(n,d)}\widehat{f}_{n,l}Y_{n,l} \right\|_p.
\end{align}
When $p=2$ it is known that $W_2^r(\mathbb{S}^{d-1})$ is a Hilbert space with the inner product:
\begin{align*}
    \langle f,g \rangle_{W_2^r(\mathbb{S}^{d-1})}=\sum_{n=0}^{\infty}(1+\lambda_n)^r\sum_{l=1}^{N(n,d)}\widehat{f}_{n,l}\widehat{g}_{n,l}.
\end{align*}
In addition, we have the following continuous embedding lemma, see \cite{hesse2006lower} and also \cite[Eq. 14, p. 420]{kamzolov1982best}.
%We state two continuous embedding lemmas at the beginning of this section which are shown in \cite{hesse2006lower,kamzolov1982best}.
 By these lemmas, we know that the infinity norm can be bounded by the Sobolev norm when $r > \frac{d-1}{p}$.

\begin{prop}\label{embedding}
 For $d \geq 3$, $1\leq p\leq \infty$, and $r > \frac{d-1}{p}$, the Sobolev space $W^{r}_p(\mathbb{S}^{d-1})$ is continuously embedded into $C(\mathbb{S}^{d-1})$, the space of continuous functions on $\sph$,
 which implies
\begin{align*}
 \left\|f\right\|_\infty \leq c_{r,d} \left\|f\right\|_{W^r_p(\mathbb{S}^{d-1})}, \qquad \forall f \in W^{r}_p(\mathbb{S}^{d-1}),
\end{align*}
where $c_{r, d}$ is a constant independent of $f$.
\end{prop}

\subsection{Near-best approximation and discretization}

The best approximation of a function by those from polynomial spaces of various degrees might be nonlinear. A useful tool in spherical harmonic analysis is a linear scheme
$L_n$.

\begin{definition}
 Given a $C^\infty\left([0,\infty)\right)$ function $\eta$ with $\eta(t) = 1$ for $0 \leq t \leq 1$ and $\eta(t) = 0$ for $t \geq 2$,
 we define a sequence of linear operators $L_n$, $n\in\ZZ_+$, on $L_p(\mathbb{S}^{d-1})$  with $1\leq p \leq \infty$ by
\begin{equation}\label{Ln_definition}
\begin{aligned}
L_n(f)(x) := \sum_{k=0}^\infty \eta\left(\f k n\right)\proj_k(f)(x)=\int_{\mathbb{S}^{d-1}} f(y)l_n(\langle x,y\rangle)d\s_d(y), \quad x\in \mathbb{S}^{d-1},
\end{aligned}
\end{equation}
where with $\lambda=\frac{d-2}{2}$, $l_n$ is a kernel given by
\begin{equation}\label{l_n}
\begin{aligned}
l_n(t) = l_{n, d}(t) := \sum_{k=0}^{2n} \eta\left(\f k n\right)  \f{\l+k}{\l} C^\l_k(t), \qquad t\in[-1,1].
\end{aligned}
\end{equation}
\end{definition}

It can be found in \cite[Chapter 4]{dai2013approximation} that $L_n$ is near best, achieving the order of best approximation for  $f\in W^r_p(\sph)$.

\begin{lem}\label{lemma:bestapproximationerror}
For $n\in \NN$, $1\leq p\leq \infty$ and $f\in W^r_p\left(\mathbb{S}^{d-1}\right)$, there holds
\begin{align}
\left\|f - L_n(f)\right\|_{p} \leq c n^{-r} \left\|f\right\|_{W^r_p(\sph)},\label{equation:bestapproximationerror}
\end{align}
where $c$ is a constant depending only on the function $\eta$ in defining $L_n$.
\end{lem}

Note that since $(-\Delta_0+I)^{-r/2}$ is self-adjoint, for $x\in\sph$,
\begin{align*}
L_n(f)(x)=&\Bl\la f, l_n(\la x,\cdot\ra)\Br\ra_{L_2(\sph)} \\
=&\Bl\la (-\Delta_0+I)^{r/2}f, (-\Delta_0+I)^{-r/2}l_n(\la x,\cdot\ra)\Br\ra_{L_2(\sph)}\\
=:&\int_{\sph} F_r(y)\zeta_{n,r}(\la x,y\ra)d\s(y),
\end{align*}
here and in what follows, we denote
$$F_r=(-\Delta_0+I)^{r/2}f \quad \hbox{and} \quad \zeta_{n,r}(\la x,\cdot\ra)=(-\Delta_0+I)^{-r/2}l_n(\la x,\cdot\ra).$$
 In fact, $\zeta_{n,r}$ is a polynomial of degree at most $2n$ with expression $$ \zeta_{n,r}(t)=\sum_{k=0}^{2n}(1+\l_k)^{-r/2} \eta\left(\f k n\right)\f{\l+k}\l C_k^\l(t), \quad t\in[-1,1]. $$
The novelty here using a fractional power of $(-\Delta_0+I)$ caused by the regularity of $f\in W^r_\infty(\mathbb{S}^{d-1})$ enables us to get an $r$-dependent error bound for discretizing $L_n(f)$: the larger the regularity index $r$, the smaller the bound. %$\zeta_{n,r}(t)=\sum_{k=0}^{2n}\eta\left(\f k n\right)(\l_k+1)^{-r/2}\f{\l+k}\l C_k^\l(t)$. }

To approximate the function $L_n(f)$ by a neural network,
we need a stepping stone, discretizing the integral form (\ref{Ln_definition}) to an empirical version
\begin{equation}\label{hatL}
\widehat{L}^{\mathbf y}_{n,m}(f)(x) = \frac{1}{m}\sum_{i=1}^m F_r(y_i) \zeta_{n,r}(\langle x, y_i \rangle), \qquad x\in \mathbb{S}^{d-1}
\end{equation}
given in terms of a sample ${\mathbf y}=\{y_1, \ldots, y_m\} \subset \sph$.
The following estimate for the error $\widehat{L}^{\mathbf y}_{n,m}(f) - L_n(f)$ will be proved by a probability inequality in Section \ref{proof of main results}.
Such a probabilistic argument has been applied in \cite{Klusowski2018}.

\begin{lem}\label{lem:sampleversion}
Let $d \geq 3$, $r>0$, and $\tau>0$.  If $f \in W^r_\infty(\mathbb{S}^{d-1})$, then for any $n, m\in\NN$,
there exist $\mathbf y=\{y_1, y_2,\ldots, y_m\}\subset \sph$ such that
\begin{align}\label{dcnn-4}
\left\|\widehat{L}^{\mathbf y}_{n,m}(f) - L_n(f)\right\|_{\infty} \leq C_{r, d, \tau} \f{\sqrt{\Lambda_{2(d -1-r+ \tau)} (n)}}{\sqrt{m}} \left\|f\right\|_{W^r_\infty(\sph)},
\end{align}
where for $\theta\in\RR$ and $n\in\NN$ we denote
$$ \Lambda_{\theta} (n) = \left\{\begin{array}{ll} n^{\theta}, & \hbox{if} \ \theta>0, \\
 \log (n+1), & \hbox{if} \ \theta =0, \\
 1, & \hbox{if} \ \theta <0, \end{array}\right. $$
and $C_{r, d, \tau}$ is a positive constant depending on $r, d, \tau$ but not on $n, m$ or $f$.
\end{lem}

\subsection{Approximating ridge functions by deep CNNs}

The last step in our spherical analysis of deep CNNs is to approximate the ridge function $\widehat{L}^{\mathbf y}_{n,m}(f)$
by functions from the network with a bound to be proved in Section \ref{proof of main results}.

\begin{lem}\label{lem:final}
Let $2 \leq S \leq d$, $d \geq 3$, $r>0$, $m, n, N\in\NN$, $f \in W^r_\infty(\mathbb{S}^{d-1})$, and ${\mathbf y}=\{y_1, \ldots, y_m\} \subset \sph$.
Let $J\geq \lceil \frac{md-1}{S-1}\rceil$. Then there exists a deep neural network consisting of $J$ layers of CNNs with filters of length $S$ and bias vectors
 satisfying (\ref{restrrow}) followed by downsampling and two fully connected layers with widths (\ref{fullwidths}), connection matrices (\ref{fullmatrices})
 and bias vectors given explicitly in (\ref{bJplus1}), (\ref{bJplus2}) below involving two parameters $B^{(J)}, B^{(J+2)}\in\RR$
 such that the hypothesis space $\mathfrak{H}_{J,\mathcal{D}_1, \mathcal D_2,S}$ contains a function $\hat f$ satisfying
\begin{equation}\label{finalbound}
\left\|\widehat{L}^{\mathbf y}_{n,m}(f)-  \hat f\right\|_{\infty}
\leq c'_{r, d} \f{n^2 \Lambda_{d-1-r} (n)}{N}\|f\|_{W^r_\infty(\sph)},
\end{equation}
where $c'_{r, d}$ is a constant depending only on $r$ and $d$.
%and with the zero vector $O \in\RR^{2N+3}$, the vector $\vec t =(-1+\f{j-2}N)_{j=1}^{2N+3}$, and a parameter $B^{(J)}$,
%the full connection matrix $F^{(J+1)}\in \RR^{\mathcal{D}\times \lfloor (d+JS)/d\rfloor}$ and the bias vector $b^{(J+1)}$ are given in block forms as
%\begin{equation}\label{fullconstruct}
%F^{(J+1)}=\left[\begin{array}{cccc}
%{\bf 1}_{2N+3} & O &O& O  \cdots  O \\
%O & {\bf 1}_{2N+3} & O&O  \cdots  O \\
%\vdots & \ddots &\ddots& \vdots\\
%O & \cdots    & O& {\bf 1}_{2N+3}
%\end{array}\right], \ b^{(J+1)}=\left[\begin{array}{c}
%B^{(J)} {\bf 1}_{2N+3} + \vec t \\
%\vdots \\
%B^{(J)} {\bf 1}_{2N+3} + \vec t \\
%(B^{(J)} +1) {\bf 1}_{\mathcal{D} -m (2N+3)}
%\end{array}\right].
%\end{equation}
The total number of free parameters $\mathcal{N}$ in the network can be bounded as
$$\mathcal{N} \leq J (3S+2)+m+2N +4.$$
\end{lem}
%\begin{remark}For any $f \in W^r_\infty(\mathbb{S}^{d-1})$, this theorem confirms the existence of a deep convolutional neural networks which can approximate function $f$ with the best approximation error with length at most $\lceil \frac{md-1}{s-1} \rceil + 1$.
%\end{remark}

\section{Comparison with Related Work}

In this section, we give a brief review of related work on rates of function approximation by neural networks.

The {\bf fully connected} shallow nets (\ref{ShallowNetworks}) or multi-layer nets (\ref{dcnn-1})
have nice approximation properties due to the fully connected nature, which
was well studied in a large literature around 30 years ago.
When the activation function is a $C^\infty$ sigmoidal type function, approximation rates were obtained
by many authors. In particular, in \cite{Barron} rates were given for functions in $f\in L_2({\RR}^d)$ whose Fourier transforms $\hat{f}$
satisfy a decay condition $\int_{\RR^d} |w| |\hat{f}(w)| d w<\infty$.
Another typical result based on localized Taylor expansions asserts \cite{Mhaskar1993} that even for shallow nets (\ref{ShallowNetworks}), we have
$\inf_{c_{k}, w_k, b_{k}} \left\|f_N - f\right\|_{C([-1, 1]^d)} =O(N^{-r/d})$ for $f \in W_\infty^r ([-1, 1]^d)$,
if for some $b\in{\mathbb R}$ and some integer $\ell \in\NN\setminus\{1\}$, the $C^\infty$ activation function
$\sigma$ satisfies $\sigma^{(k)} (b)\not= 0$ for all $k\in\ZZ_+$ and
$\lim_{u\to-\infty} \sigma(u)/|u|^\ell=0$ and $\lim_{u\to \infty} \sigma(u)/u^\ell=1$.
These conditions required by the localized Taylor expansion approach are not satisfied by ReLU,
so the approximation theory developed 30 years ago does not apply to ReLU.
The difficulty was overcome in the recent deep learning literature and approximation properties of ReLU nets
were established in \cite{Klusowski2018} for ReLU shallow nets and functions satisfying $\int_{\RR^d} |w| |\hat{f}(w)| d w<\infty$,
in \cite{Yarosky, Grohs, Petersen, Nakada2019} for deep nets and functions from $W_\infty^r ([-1, 1]^d)$ with $0< r \leq 2$,
and in \cite{Shaham} for approximation on manifolds.
These results are obtained for fully connected nets.

Deep CNNs are different from fully connected nets. They have special sparse convolutional connection matrices (\ref{Toeplitz}),
which leads to sparsity and reduces the computational complexity for structured data.
Recently in \cite{zhou2020universality}, for functions $f$ on $\Omega \subset[-1,1]^d$
satisfying $f = F\vert_\Omega$ with $F \in W^r_2\left(\mathbb{R}^d\right)$ and an integer index $r>2+d/2$,
it was shown that the approximation accuracy $\|f-\hat{f}\|_{\infty} \leq \epsilon$
can be achieved by a deep CNN of depth $4 \lceil \frac{1}{\epsilon^2} \log  \frac{1}{\epsilon^2}\rceil$ and at most
$\lceil \frac{75}{\epsilon^2} \log  \frac{1}{\epsilon^2}\rceil d$ free parameters.
The linear increment of the free parameter number with respect to $d$ improves
the bound in Theorem 1 of \cite{Yarosky} which requires at least $2^d \epsilon^{-d/r}$ free parameters
and $\frac{C_0 d}{4} (\log (1/\epsilon) + d)$ fully connected layers with $C_0 >0$ to achieve the same approximation accuracy $\epsilon$.
Periodized deep CNNs with different architectures and connection matrices different from the Toeplitz convolutional ones (\ref{Toeplitz}) were
shown in \cite{Oono2019, Petersen2020} to be able to realize the output layer of any fully-connected DNN with free parameters of the same order.
The same result was shown for deep CNNs (\ref{InitialLayers}) in \cite{zhou2020theory}.

The index $r>2+d/2$ required in \cite{zhou2020universality} can be very large for processing high dimensional data.
Hence approximated functions are required to possess high regularity which is not the usual case in applications. The problem happens because the approximation considered in \cite{zhou2020universality} is measured in the $L_\infty$ norm, while the Sobolev space $H^r(\mathbb{R}^d)$ requires derivatives of various orders to belong to the $L_2$ space.
This essentially causes the technical difficulty
 by embedding $W_2^r ([-1, 1]^d)$ into
$W^s_\infty ([-1, 1]^d)$ which requires $s<r-d/2$.
To overcome the difficulty, we consider the case when the data is from the unit sphere $\sph$.
 The restriction can be relaxed in this situation through applying spherical harmonic expansions to construct a near best approximation for functions $f \in W^{r}_\infty(\sph)$ in $L_\infty$ norm, while the Sobolev embedding theorem is only used in discretizing integrals.
In the literature, there have been some other harmonic analysis
approaches in dealing with approximation by fully connected neural networks, using ridgelet transforms in \cite{Sonoda},
local Taylor expansions in \cite{Yarosky}, and B-spline functions in \cite{Suzuki2019}.
Our spherical harmonic analysis approach makes full use of the inner product nature (\ref{rkhsform}) of the reproducing kernel of $\CH_n^d$
which, after discretizing the polynomial approximation $L_n (f)$ to $\widehat{L}^{\mathbf y}_{n,m}(f)$, enables us to
represent the linear transformations $\{\langle y_i, x \rangle\}$ by deep CNNs with linearly increasing widths, an idea
borrowed from our earlier work \cite{zhou2020universality}.
A key consequence of our approach is to allow the index $r$ here to be an arbitrarily small positive number,
which relaxes the restriction in \cite{zhou2020universality}
for the regularity of the approximated function. While the approximation of non-smooth functions is unified for $r>0$ and the same order
$\OO \left(\epsilon^{-2}\right)$ for the number of network free parameters to achieve an approximation accuracy $\epsilon>0$ is kept when $r>d-1$,
a parameter number of order $\OO \left(\epsilon^{-\frac{2d}{r}}\right)$ is required when $0< r <d-1$. This is due to our approach of
using a Hilbert space $W^s_2 (\sph)$ in our probabilistic estimate for the discretization,
which makes our rate suboptimal compared with \cite{Yarosky, Nakada2019, Schmidt-Hieber} for $r<d-1$.
It would be interesting to derive optimal rates of approximating by deep CNNs functions from $W^r_\infty (\sph)$ with small $r$.

On the other hand, as stated in Remark \ref{accuracyThm2}, for an additive ridge function on $\sph$ in the family (\ref{additiveridge}),
deep CNNs followed by a fully connected layer can extract linear features $\{y_j\}_{j=1}^m$ and then approximate the function efficiently,
with the same order of network free parameters as that for approximating a univariate function by fully connected DNNs.
This demonstrates the superiority of deep CNNs in approximating functions with structures.
It would be of great interest to explore other structures of multivariate functions for which deep CNNs
together with network architectures like pooling and parallel channels may have super performance in function approximations and representations.
Applying deep CNNs to some practical or empirical problems involving additive ridge functions (\ref{additiveridge}) would also help
understand advantages of convolutional structures of deep learning in some practical domains.

\section{Proof of the Main Results}\label{proof of main results}

This section is devoted to the proof of our main analysis. Our analysis for the error $f-\hat{f}$
is carried out by means of the bounds for $\|f-L_n(f)\|_\infty$ in \lemref{lemma:bestapproximationerror},
$\|L_n(f) - \widehat{L}_{n,m}^{\mathbf y}(f)\|_\infty$ in \lemref{lem:sampleversion}, and
$\|\widehat{L}^{\mathbf y}_{n,m}(f)- \hat f\|_{\infty}$ in \lemref{lem:final}.

\subsection{Proving the lemma on discretization}

To complete our analysis, we first prove \lemref{lem:sampleversion} and
\lemref{lem:final}.

The proof of \lemref{lem:sampleversion} is based on the following probability inequality for random variables with values in a Hilbert space
which can be found in \cite{smale2007learning}.

\begin{lem}\label{lemma:Hoeffding}
Let $(H, \|\cdot\|)$ be a Hilbert space and $\xi$ be a random variable on $\left(Y, \rho\right)$ with values in $H$. Assume $\left\|\xi\right\| \leq M < \infty$ almost surely. Denote $\sigma^2(\xi) = E\left(\left\|\xi\right\| ^2\right)$. Let $\left\{y_i\right\}_{i=1}^m$ be independent random drawers of $\rho$. Then for any $0< \delta <1$, we have with confidence $1-\delta$,
\begin{equation*}
\left\|\frac{1}{m} \sum_{i=1}^m \xi(y_i) -E (\xi)\right\|_H
\leq \frac{2M\log(\frac{2}{\delta})}{m} + \sqrt{ \frac{2\sigma^2(\xi)\log(\frac{2}{\delta})}{m} }.
\end{equation*}
\end{lem}

\begin{proof}[Proof of \lemref{lem:sampleversion}]
Recall that $L_n(f)$ is defined by \eqref{Ln_definition} with $f \in W^r_\infty(\sph)$ and $\tau>0$. In applying \lemref{lemma:Hoeffding} we take the Sobolev space $W_2^s(\sph)$
with the smoothness index $s=\tau+\f{d-1}2$ to be the Hilbert space $H$ and the random variable $\xi$ on $\left(\mathbb{S}^{d-1}, \s_d\right)$ with values in $H$ given by
$$\xi(y) = F_r(y) \sum_{k=0}^{2n} (1+\l_k)^{-r/2}\eta\left(\frac{k}{n}\right) Z_k (y, \cdot)\in H, \qquad  y\in \sph.$$
%in $W^s_2 \left(\mathbb{S}^{d-1}\right)$ where $d\geq 3$, $s > 0\frac{d-1}{2}$.
Then $E(\xi) = L_n (f)$ and $\frac{1}{m} \sum_{i=1}^m \xi(y_i)= \widehat{L}^{\mathbf y}_{n,m}(f)$.

To bound the norm $\|\xi\|=\|\xi\|_{W^s_2}$, we recall the norm of $W^s_2(\sph)$ given by \eqref{Sobolev_norm} with $p=2$ and find for $y\in \sph$,
$$
\left\| \xi(y) \right\|_{W^s_2(\sph)}  =  \left\|F_r(y)\sum_{k=0}^{2n}\left(1+\lambda_k\right)^{\frac{s-r}{2}}\eta\left(\frac{k}{n}\right) Z_k ( y, \cdot)\right\|_{L_2(\sph)},
$$
where $ \lambda_k = k(k+d-2)$. Then by the orthogonality and reproducing properties,
\begin{eqnarray*}
\left\| \xi(y) \right\|_{W^s_2(\sph)}^2
&=& \left(F_r(y)\right)^2  \sum_{k=0}^{2n}   \left(1+\lambda_k\right)^{s-r} \eta^2\left(\frac{k}{n}\right) Z_k(y,y)\\
&=& \left(F_r(y)\right)^2 \sum_{k=0}^{2n} \left(1+\lambda_k\right)^{s-r} \eta^2\left(\frac{k}{n}\right) N(k,d),
\end{eqnarray*}
where we have used the identity $Z_k(y,y) =N(k,d)$ found in \cite{dai2013approximation} as Corollary 1.2.7 and $N(k,d)$ is the dimension of spherical harmonics $\mathcal H^d_k$.
%Noting
%\begin{align*}
%N(k,d) = {{k+d-1}  \choose {k}}-{{k+d-3}\choose {k-2}}= \OO_d (k^{d-2}),
%\end{align*} and
Notice that for $k\in\NN$, $k^2 < 1+\lambda_k \leq d k^2$. We find $\left(1+\lambda_k\right)^{s-r} \leq d^{\max\{s-r, 0\}} k^{2(s-r)}$ for either $s-r \geq 0$ or $s-r<0$.
Since $0\leq \eta(t)\leq 1$ for $t\in[0,2]$,
we can apply \eqref{H_dimension} to estimate the summation as
$$\sum_{k=0}^{2n} \left(1+\lambda_k\right)^{s-r} \eta^2\left(\frac{k}{n}\right) N(k,d) \leq 1 + \sum_{k=1}^{2n} d^{\max\{s-r, 0\}} k^{2(s-r)} c'_d k^{d-2} $$
with a constant $c'_d$ depending only on $d$, while
$$ \sum_{k=1}^{2n} k^{2(s-r)+d-2} \leq 1 + \left\{\begin{array}{ll}
\frac{3^{2(s-r)+d-1}}{2(s-r)+d-1} n^{2(s-r)+d-1}, & \hbox{if} \ 2(s-r)+d-2 >-1, \\
1+ \log (n+1), & \hbox{if} \ 2(s-r)+d-2 =-1, \\
\frac{1}{1-2(s-r)-d}, & \hbox{if} \ 2(s-r)+d-2 <-1.
\end{array}\right.
$$
Combining this with the definitions of the norm $\|f\|_{W^r_\infty(\sph)}$ and the function $\Lambda_\tau (n)$, we know that
$\left\| \xi(y) \right\|_{W^s_2}^2$ can be bounded as
\begin{equation*}
\left\| \xi(y) \right\|_{W^s_2(\sph)}^2
\leq c_{s, r, d}^2 \|f\|_{W^r_\infty(\sph)}^2 \Lambda_{2s-2r +d -1} (n),
\end{equation*}
where $c_{s, r, d}$ is a positive constant independent of $f$ or $n$. Thus the random variable $\xi$ satisfies
the condition $\left\|\xi\right\| \leq M < \infty$ in \lemref{lemma:Hoeffding} with $M= c_{s, r, d} \|f\|_{W^r_\infty(\sph)} \sqrt{\Lambda_{2s-2r +d -1} (n)}$.
So by \lemref{lemma:Hoeffding} with $\delta = \frac{1}{2}$ and $\sigma^2(\xi) \leq M^2$, we know from the positive measure of the sample set that there exists a set of points $\mathbf y=\left\{y_i\right\}_{i=1}^{m} \in \mathbb{S}^{d-1}$ such that
\begin{eqnarray*}
\left\|\frac{1}{m} \sum_{i=1}^m \xi(y_i) -E (\xi)\right\|_H &=& \left\|\widehat{L}_{n,m}^{\mathbf y}(f)-L_n(f)\right\|_{W_2^s(\sph)} \\
&\leq& \f{6c_{s, r, d} \|f\|_{W^r_\infty(\sph)} \sqrt{\Lambda_{2s-2r +d -1} (n)}}{\sqrt{m}}.
\end{eqnarray*}
This verifies \eqref{dcnn-4} by the embedding \propref{embedding} with $p=2$ and $s=\tau+\f{d-1}2 > \f{d-1}2$.
\end{proof}

\subsection{Proving the lemma on ridge approximation}

The proof of \lemref{lem:final} about approximating the function $\widehat{L}^{\mathbf y}_{n,m}(f)$
is conducted by approximating the ridge functions $l_n(\langle y_i, x \rangle)$ in (\ref{hatL}) with $y_i \in \mathbb{S}^{d-1} \subset \RR^d$ by deep CNNs.
A key idea in our analysis is to use the inner product form (\ref{rkhsform}) of the reproducing kernel of $\CH_n^d$
and then to apply convolutional factorizations to realize the generated linear features, which enables us to
conduct analysis after removing the restriction on large regularity index $r$.
This idea might be applied to some other learning theory problems \cite{FanHuWuZhou, GXGZ, LinZhou, ZhouDist}.

We first apply the following two lemmas proved in \cite{zhou2020universality} implying that the linear function $\langle y_i, x \rangle =y_i \cdot x$
can be realized by deep CNNs by factorizing $y_i$ regarded as a sequence into convolutions of filters supported in $\{0, 1, \ldots, S\}$.

\begin{lem}\label{convolutionalfactorization}
Let $S \geq 2$ and $W = \left(W_k\right)_{k=-\infty}^{\infty}$ be a sequence supported in $\left\{0, \cdots, \mathcal{M}\right\}$ with $\mathcal{M} \geq 0$.
Then there exists a finite sequence of filters $\left\{w^{(j)}\right\}_{j=1}^p$ each supported in $\left\{0,\cdots, S\right\}$ with $p\leq \lceil \frac{\mathcal{M}}{S-1}\rceil$ such that the following convolutional factorization holds true
\begin{equation*}
W = w^{(p)} \ast w^{(p-1)} \ast \cdots \ast w^{(2)} \ast w^{(1)}.
\end{equation*}
\end{lem}

\begin{lem}\label{productofmatrix}
Let $\{w^{(k)}\}_{k=1}^J$ be a set of sequences supported in $\{0,1,\ldots, S\}$. Then
\begin{equation}
\begin{aligned}
T^{(J)} \cdots  T^{(2)}  T^{(1)} = T^{\left(J,1\right)} :=\left(W_{i-k}\right)_{i=1, \dots, d+JS, k=1, \dots, d} \in \RR^{(d+JS) \times d}
\end{aligned}
\end{equation}
is a Toeplitz matrix associated with the filter $W=w^{(J)} \ast \cdots \ast w^{(2)} \ast w^{(1)}$ supported in $\left\{0, 1, \cdots, JS\right\}$.
\end{lem}

We then construct a fully connected layer to approximate the univariate function $l_n$ by continuous piecewise linear functions (splines)
spanned by $\{\s(\cdot-t_i)\}^N_{i=1}$ with $t_i = -1+\f{i-2}N$,
based on the following well known result in approximation by splines which can be found in \cite{DevoreLorentz} and \cite[Lemma 6]{zhou2018deepdistributed}.

\begin{lem}\label{spline}
Given an integer $N$, let $\textbf{t}=\{t_i\}_{i=1}^{2N+3}$ be the uniform mesh on $\left[-1-\frac{1}{N}, 1+\frac{1}{N}\right]$ with $t_i=-1+\f{i-2}N$.
 Construct a linear operator $L_{\textbf{t}}$ on $C[-1,1]$ by
\[L_{\textbf t}(f)(u)=\sum_{i=2}^{2N+2} f(t_i)\d_i(u), \quad u\in [-1,1], \, f\in C[-1,1],\]
where $\d_i\in C(\RR)$, $i=2,\ldots, 2N+2$, is given by
\begin{equation}\label{delta}\delta_{i}(u)=N(\sigma\left(u-t_{i-1}\right)-2\sigma\left(u-t_{i}\right)+ \sigma\left(u-t_{i+1}\right)).
\end{equation}
Then for $g\in C[-1,1]$, $\left\|L_{\mathbf{t}}(g)\right\|_{C\left[-1,1\right]} \leq \left\|g\right\|_{C\left[-1,1\right]}$ and
$$\left\|L_{\mathbf{t}}(g)-g\right\|_{C\left[-1,1\right]} \leq 2 \omega\left(g, 1/N\right)$$
%and for $g\in L_p[-1,1]$ with $1\leq p <\infty$,
%$$\left\|L_{\mathbf{t}}(g)-g\right\|_{L_p[-1,1]} \leq 2 \omega\left(g,1/N\right)_p,$$
%where $\omega(g, \mu)_p$ is the modulus of continuity of $g$ given by
where $\omega(g, \mu)$ is the modulus of continuity of $g$ given by
\[\omega(g, \mu)=\sup_{|t|\leq \mu} \Bl\{|g(v)-g(v+t)|: v, v+t \in[-1,1]\Br\}.\]
\end{lem}

For the convenience of counting free parameter numbers, we introduce a linear operator ${\mathcal L}_N: \RR^{2N+1} \to \RR^{2N+3}$ given for $\zeta =(\zeta_i)_{i=1}^{2N+1} \in \RR^{2N+1}$ by
\begin{equation}\label{diffoperator}
\left({\mathcal L}_N (\zeta)\right)_i =
\begin{cases}
\zeta_2,                            &\text{for}~i=1,\\
\zeta_3 -2\zeta_2,                &\text{for}~i=2,\\
\zeta_{i-1} -2\zeta_i +\zeta_{i+1},   &\text{for}~3\leq i \leq 2N+1,\\
\zeta_{2N+1} -2\zeta_{2N+2},        &\text{for}~i=2N+2,\\
\zeta_{2N+2},                       &\text{for}~i=2N+3.
\end{cases}
\end{equation}
An important property of the operator ${\mathcal L}_N$ is to express the approximation operator $L_{\textbf{t}}$ on $C[-1,1]$ in terms of $\{\sigma\left(\cdot-t_{j}\right)\}_{j=1}^{2N+3}$ as
\begin{equation}\label{splineiden}
L_{\textbf t}(f)=N \sum_{i=1}^{2N+3} \left({\mathcal L}_N \left(\left\{f(t_k)\right\}_{k=2}^{2N+2}\right)\right)_i \sigma\left(\cdot-t_{i}\right), \quad \forall f\in C[-1,1].
\end{equation}

\begin{proof}[Proof of \lemref{lem:final}]
 For $m\in\NN$ and $\mathbf y=\{y_1,\ldots, y_m\}\subset \sph$,
 %To prove the last critical part,
 we take $W$ to be a sequence supported in $\left\{0,\cdots,md-1\right\}$ given by $W_{(j-1)d+(d-i)} = (y_{j})_i$
 where $j\in \left\{1,\cdots,m\right\}$ and $i \in \left\{1,\cdots,d\right\}$.
 By \lemref{convolutionalfactorization} with $\mathcal{M} = md-1$, there exists a sequence of filters $\textbf{w} = \left\{w^{(j)}\right\}_{j=1}^{J}$
 supported in $\left\{0,\cdots, S\right\}$ with $J \geq \lceil \frac{\mathcal{M}}{S-1}\rceil$ satisfying the convolutional factorization
 $W = w^{(J)} \ast w^{(J-1)} \ast \cdots \ast w^{(2)} \ast w^{(1)}$. Here for $j=p+1, \ldots, J$, we have taken $w^{(j)}$ to be the delta sequence $\delta_0$ given by $\left(\delta_0\right)_0 =1$ and $\left(\delta_0\right)_k =0$ for $k\in\ZZ\setminus \{0\}$. By \lemref{productofmatrix}, we have
$$ T^{(J)} T^{(J-1)}\cdots T^{(1)} = T^{\left(J,1\right)} =\left(W_{i-k}\right)_{i=1, \dots, d+JS, k=1, \dots, d} \in \RR^{(d+JS) \times d}, $$
where $T^{(j)}$ is the Toeplitz matrix with filter $w^{(j)}$ for $j=1,2,\ldots, J$.

Now we construct bias vectors in the neural networks.
We denote $\left\|w\right\|_1 = \sum_{k=-\infty}^{\infty}|w_k|$. Take $b^{(1)} = - \left\|w^{(1)}\right\|_1 \textbf{1}_{d_0}$ and
\begin{equation}\label{biasvector}
\begin{aligned}
b^{(j)} =  \left( \Pi^{j-1}_{p=1} \left\|w^{(p)}\right\|_1\right) T^{(j)} \textbf{1}_{d_{j-1}}
-  \left( \Pi^{j}_{p=1} \left\|w^{(p)}\right\|_1\right) \textbf{1}_{d_{j-1}+S},
\end{aligned}
\end{equation}
for $j=2,\cdots,J$. The bias vectors satisfy $b^{(j)}_{S+1} = \ldots = b^{(j)}_{d_j -S}$. Observe that $\left\| x\right\|_{\infty} \leq 1$ for $x\in\sph$.
Denote $\|h\|_\infty =\max\{\|h_j\|_\infty: j=1, \ldots, q\}$ for a vector of functions $h: \sph \to \RR^q$. We know that for $h: \sph \to \RR^{d_{j-1}}$,
 $$\left\|T^{(j)}h\right\|_\infty \leq \left\|w^{(j)}\right\|_1 \left\|h\right\|_\infty. $$
Hence the components of $h^{(J)}(x)$ satisfy
$$
\left(h^{(J)}(x)\right)_{kd}=\left\langle y_k, x \right\rangle + B^{(J)}, \qquad k=1, \ldots, m, $$
where $B^{(J)} = \Pi^{J}_{p=1} \left\|w^{(p)}\right\|_1$. Applying the downsampling operator \eqref{downsample} leads to
\[\mathfrak{D}_d \left(h^{(J)}(x)\right)=
\begin{bmatrix}
\left\langle y_1, x \right\rangle\\
\vdots\\
\langle y_m,  x \rangle\\
0\\
\vdots\\
0
\end{bmatrix}
+ B^{(J)}\textbf{1}_{\lfloor (d+JS)/d\rfloor}.\]

Denote $\widehat{d} = \lfloor (d+JS)/d\rfloor$. Since $J \geq \lceil \frac{md-1}{S-1}\rceil$, we have
$$\frac{d+JS}{d} \geq 1 + \frac{md -1}{d} \frac{S}{S-1} >  1 + \frac{md -1}{d} \geq m. $$
Hence $\widehat{d} \geq m$.

We turn to expressing the last two fully connected layers. Of them, $h^{(J+1)}$ is given by
\[h^{(J+1)}(x)=\s(F^{(J+1)}\mathfrak{D}_d(h^{(J)}(x))-b^{(J+1)})\]
with the connection matrix $F^{(J+1)} =\Xi_{\mathcal{D}_2, \textbf{1}_{2N+3}}$ stated in (\ref{fullmatrices})
and the bias vector
\begin{equation}\label{bJplus1}
b^{(J+1)}_{(j-1)(2N+3) +i}=
\begin{cases}
B^{(J)} + t_{i}, ~~~&\text{if}~j=1, \ldots, m, \ i=1, \ldots, 2N+3, \\%\ell=(j-1)(2N+3)+k, k=0,1,\ldots, 2N+2;\\
B^{(J)} + 1, ~~~&\text{if}~j >m,
\end{cases}
\end{equation}
where $\textbf{t} := \left\{t_1 < \cdots < t_{2N+3}\right\}$ is given in \lemref{spline}.  Note that $F^{(J+1)}$ is a determined matrix without free parameters.
Then the first fully-connected layer $h^{(J+1)}(x)\in\RR^{\hat d(2N+3)}$ of the deep network is
\begin{equation}\label{firstfull}
\left( h^{(J+1)} \right)_{(j-1)(2N+3) +i} =\left\{
                    \begin{array}{ll}
                      \sigma \left( \langle y_j, \cdot \rangle - t_i\right), & \hbox{if}~j\leq m, \ 1 \leq i \leq 2N+3, \\
                      0, ~~~&\text{if}~j >m.
                    \end{array}
                  \right.
\end{equation}

Write $h^{(J+1)}(x)\in\RR^{\hat d(2N+3)}$ in a block form with $\hat d$ blocks of equal size $2N+3$, then the $j$-th block is
$\left[\sigma \left( \langle y_j, x \rangle - t_i\right)\right]_{i=1}^{2N+3}$ for $j=1, \ldots, m$, while the other blocks are zero vectors.

Take the vector $\Theta_N \in\RR^{2N+3}$ in the connection matrix $F^{(J+2)} = \Xi_{\mathcal{D}_2, \Theta_N}^T$ of the second fully-connected layer stated in
(\ref{fullmatrices}) in terms of the linear operator ${\mathcal L}_N$ as
$$ \Theta_N = {\mathcal L}_N \left(\left\{\zeta_{n,r}(t_i)\right\}_{i=2}^{2N+2}\right), $$
then by the identity (\ref{splineiden}),
for $j=1, \ldots, m$, the $j$th entry of the product $F^{(J+2)}h^{(J+1)}(x)$ equals
\begin{eqnarray*}
\T_N^T \left[\sigma \left( \langle y_j, x \rangle - t_i\right)\right]_{i=1}^{2N+3}
&=&\sum_{i=1}^{2N+3} \left({\mathcal L}_N \left(\left\{\zeta_{n,r}(t_i)\right\}_{i=2}^{2N+2}\right)\right)_i \sigma \left( \langle y_j, x \rangle - t_i\right) \\
&=& \frac{1}{N}L_{\textbf t}\left(\zeta_{n,r}\right)\left(\langle y_j, x \rangle\right).
\end{eqnarray*}
The other entries of the product $F^{(J+2)}h^{(J+1)}(x)$ vanish. Thus, by taking $B^{(J+2)}=\|\zeta_{n,r}\|_{C[-1,1]}$ and
\begin{equation}\label{bJplus2}
b^{(J+2)}= \left[\begin{array}{c}
- \frac{B^{(J+2)}}{N} {\bf 1}_{m} \\
O
\end{array}\right],
\end{equation}
we see from the homogenous property $\sigma(u/N) =\sigma(u)/N$ that the last layer $h^{(J+2)}$ is given by
\[ h^{(J+2)}(x)=\frac{1}{N}\left[\begin{array}{c}
\left[L_{\textbf t}\left(\zeta_{n,r}\right)\left(\langle y_j, x \rangle\right) + B^{(J+2)}\right]_{j=1}^m \\
O
\end{array}\right]. \]
For the coefficients we choose $c^{(J+2)}\in \RR^{\hat d}$ as
\[c^{(J+2)}_j=\left\{
                \begin{array}{ll}
                 \frac{N}{m} F_r(y_j), & \hbox{if $j=1,\ldots,m$,} \\
                  0, & \hbox{otherwise}
                \end{array}
              \right.\]
and $ A=B^{(J+2)}\frac{1}{m}  \sum_{j=1}^m F_r(y_j)$.
Then we have that for $x\in\sph$,
\begin{align}
&\left|\widehat{L}^{\mathbf y}_{n,m}(f)(x) - c^{(J+2)}\cdot h^{(J+2)}(x)-A\right| \nonumber \\
=&\left|\f 1 m\sum_{j=1}^m F_r(y_j) \zeta_{n,r}(\la y_j,x\ra)-\f 1 m\sum_{j=1}^m F_{r}(y_j)L_{\textbf t}\left(\zeta_{n,r}\right)\left(\langle y_j, x \rangle\right)\right| \nonumber\\
\leq &\|f\|_{W^{r}_\infty(\sph)} \|\zeta_{n,r}-L_{\mathbf t}(\zeta_{n,r})\|_{C[-1,1]}. \label{dcnn-3}
\end{align}
Since $\zeta_{n,r}$ is an algebraic polynomial of degree at most $2n$, by Markov's inequality,
\[\|\zeta_{n,r}'\|_{C[-1,1]}\leq (2n)^2\|\zeta_{n,r}\|_{C[-1,1]}. \]
Combining this with the bound $\|\zeta_{n,r}\|_{C[-1,1]}\leq\sum_{k=1}^{2n}k^{-r} N(k,d)$ followed from Corollary 1.2.7 of \cite{dai2013approximation}, we know that
 \[\|\zeta'_{n,r}\|_{C([-1,1])}\leq c_d  n^2 \sum_{k=1}^{2n} k^{d-2-r},
 \]
where $c_d$ is a constant depending only on $d$. But
$$\sum_{k=1}^{2n} k^{d-2-r} \leq 1 + \left\{\begin{array}{ll}
\frac{3^{d-1-r}}{d-1-r} n^{d-1-r}, & \hbox{if} \ d -2 -r >-1, \\
1 + \log (n+1), & \hbox{if} \ d -2 -r =-1, \\
\frac{1}{r+1-d}, & \hbox{if} \ d -2 -r <-1,
\end{array}\right.$$
which is bounded by $c''_{r, d} \Lambda_{d-1-r} (n)$ with a positive constant $c''_{r, d}$ depending only on $r$ and $d$.
It follows that $\omega(\zeta_{n,r},\f 1 N)\leq  c_d c''_{r, d} n^2 \Lambda_{d-1-r} (n)/N$.
 Combining this with \eqref{dcnn-3}, \lemref{spline} and the embedding \propref{embedding} yields
\begin{align*}\label{dcnn-5}
\left\|\widehat{L}^{\mathbf y}_{n,m}(f)(x) - c^{(J+2)}\cdot h^{(J+2)}(x)-A\right\|_\infty\leq  c'_{r, d} \f{n^2 \Lambda_{d-1-r} (n)}{N}\|f\|_{W^r_\infty(\sph)},
\end{align*}
where $c'_{r, d}$ is a constant depending only on $r$ and $d$.

The total number of free parameters $\mathcal{N}$ in our network is the sum of $J(S+1)$ contributed by $\textbf{w}$,
$J(2S+1)$ by the bias vectors in the first $J$ layers, $2N+1$ contributed by the vector $\left\{\zeta_{n,r}(t_i)\right\}_{i=2}^{2N+2}$ in choosing $\Theta_N$, $2$ by the parameters $B^{(J)}$, $B^{(J+2)}$ in the fully-connected layers, and at most $m+1$ by $c^{(J+2)}$ and $A$. So it can be bounded as
$$\mathcal{N} \leq J(S+1)+J(2S+1)+2N+1+2+m+1 \leq J (3S+2)+m+2N+4.$$
This proves \lemref{lem:final}.
\end{proof}

\subsection{General error bounds}

With the proved bounds for $\|f-L_n(f)\|_\infty$ in \lemref{lemma:bestapproximationerror},
$\|L_n(f) - \widehat{L}_{n,m}^{\mathbf y}(f)\|_\infty$ in \lemref{lem:sampleversion}, and
$\|\widehat{L}^{\mathbf y}_{n,m}(f)- \hat f\|_{\infty}$ in \lemref{lem:final},
the following bounds for the error $f-\hat{f}$ follows immediately.

\begin{thm}\label{thm:general}
Let $2 \leq S \leq d$, $d \geq 3$, $r>0$, $\tau>0$, $m, n, N\in\NN$ and $f \in W^r_\infty(\mathbb{S}^{d-1})$.
Let $J\geq \lceil \frac{md-1}{S-1}\rceil$, $\mathcal{D}_1=(2N+3)\lfloor (d+JS)/d\rfloor$ and $\mathcal{D}_2=\lfloor (d+JS)/d\rfloor$.
Then for the network constructed in \lemref{lem:final} there exists a function $\hat f\in \mathfrak{H}_{J,\mathcal{D}_1,\mathcal D_2,S}$
such that
\begin{equation}\label{generalbound}
\left\|f-  \hat f\right\|_{\infty}
\leq C'_{r, d, \tau} \left(n^{-r} +
 \f{\sqrt{\Lambda_{2(d -1-r+ \tau)} (n)}}{\sqrt{m}}  +
\f{n^2 \Lambda_{d-1-r} (n)}{N}\right)\|f\|_{W^r_\infty(\sph)},
\end{equation}
where $C'_{r, d, \tau}$ is a constant depending only on $r, d, \tau$. Moreover,
the total number of free parameters $\mathcal{N}$ in the network can be bounded as
$$\mathcal{N} \leq J (3S+2)+ m+2N+4.$$
\end{thm}

\subsection{Proving the main results}

We are in a position to derive our main results from the general error bounds in \thmref{thm:general}.

\begin{proof}[Proof of \thmref{theorem:mainresult1}]
Since $J \geq \frac{d-1}{S-1}$, we know that $\frac{(S-1)J +1}{d} \geq 1$.
Take $m= \lfloor \frac{(S-1)J +1}{d}\rfloor$.
Then $m\in\NN$ and $md-1 \leq (S-1)J$. Hence the requirement $J\geq \lceil \frac{md-1}{S-1}\rceil$ in \thmref{thm:general} is valid.

Now we take $n, N$ as
$$ \left\{\begin{array}{ll}
n= \lfloor m^{\frac{1}{2(d-1+\tau)}} \rfloor \ \hbox{and} \ N=n^{d+1}, & \hbox{if} \ 0<r< d-1, \\
n= \lfloor m^{\frac{1}{2r}} \rfloor \ \hbox{and} \ N=\lfloor n^{2+r} \rfloor, & \hbox{if} \ 0<\tau < r-(d-1).
\end{array}\right.$$
Then we know by \thmref{thm:general}
that with $\mathcal{D}_1=(2N+3)\lfloor (d+JS)/d\rfloor$ and $\mathcal D_2=\lfloor (d+JS)/d\rfloor$,
there exists a network constructed in \lemref{lem:final} containing a function $\hat f\in \mathfrak{H}_{J,\mathcal{D}_1,\mathcal D_2,S}$
such that
$$
\left\|f-  \hat f\right\|_{\infty}
\leq \left(2^{r+2}+1\right) C'_{r, d, \tau} m^{-\min\left\{\frac{r}{2(d-1+\tau)}, \frac{1}{2}\right\}} \|f\|_{W^r_\infty(\sph)}.
$$
But $m\geq \frac{(S-1)J +1}{2d} > \frac{(S-1)J}{2d}$. So we have
$$
\left\|f-  \hat f\right\|_{\infty}
\leq C_{r, d, \tau, S} J^{-\min\left\{\frac{r}{2(d-1+\tau)}, \frac{1}{2}\right\}} \|f\|_{W^r_\infty(\sph)}.
$$
with the constant
$$C_{r, d, \tau, S} :=\left(2^{r+2}+1\right) C'_{r, d, \tau} \left(2d/(S-1)\right)^{\min\left\{\frac{r}{2(d-1+\tau)}, \frac{1}{2}\right\}} . $$
This yields the desired error bound.

Observe that $m= \lfloor \frac{(S-1)J +1}{d}\rfloor \leq \frac{S}{d} J \leq J$ and
$$ N\leq \left\{\begin{array}{ll}
m^{\frac{d+1}{2(d-1+\tau)}} \leq J^{\frac{d+1}{2(d-1+\tau)}}, & \hbox{if} \ 0< r <d-1, \\
n^{2+r} \leq m^{\frac{2+r}{2r}}  \leq J^{\frac{1}{2} + \frac{1}{r}}, & \hbox{if} \ 0<\tau < r-(d-1).
\end{array}\right.$$
But $d\geq 3$ implies $\frac{d+1}{2(d-1+\tau)}<1$. In the case $0<\tau < r-(d-1)$ which implies $r> d-1 +\tau >2$, we also have
$\frac{1}{2} + \frac{1}{r}< 1$. So the total number of free parameters $\mathcal{N}$ in the network can be bounded as
$$\mathcal{N} \leq
\left(3 S+5\right) J +4.$$
The proof of \thmref{theorem:mainresult1} is complete.
\end{proof}

\begin{remark}
When $r=d-1$, from the above proof, we can see by taking $N=\lfloor n^{d+1} \log (n+1)\rfloor$ that
the statement of \thmref{theorem:mainresult1} still holds except that
the bound for the number of free parameters should be replaced by $\mathcal{N} \leq
\left(3 S+3\right) J + 2 J^{\frac{d+1}{2(d-1+\tau)}} \log (J+1) +4.$ Note that $\mathcal{N} =\OO (J)$.
So to achieve the approximation accuracy $\epsilon>0$, the depth and the number of free parameters of the network
are of orders $\OO \left(\epsilon^{-2-\frac{2}{d-1}\tau}\right)$.
\end{remark}

\begin{proof}[Proof of \thmref{thm2}]
We follow the proof of \lemref{lem:final} and construct deep CNNs of depth $J= \left\lceil \frac{md-1}{S-1}\right\rceil$
with the $m$ features $\{y_j\in\sph\}_{j=1}^m$ in the additive ridge form (\ref{additiveridge}) of the approximated function $f$,
followed by downsampling and one fully-connected layer which produces $h^{(J+1)}(x) \in \RR^{\hat{d}(2N+3)}$ expressed by (\ref{firstfull}).
Then by making use of the univariate functions $\{g_j\}_{j=1}^m$ in the additive ridge form (\ref{additiveridge}) of the approximated function $f$,
we choose the coefficient vector $c^{(J+1)}\in \RR^{\hat{d}(2N+3)}$ by means of the linear operator ${\mathcal L}_N$ as
$$ \left\{\left(c^{(J+1)}\right)_{(j-1)(2N+3) +i}\right\}_{i=1}^{2N+3} = N{\mathcal L}_N \left(\left\{g_j (t_i)\right\}_{i=2}^{2N+2}\right), \qquad j=1, \ldots, m $$
and $\left(c^{(J+1)}\right)_{(j-1)(2N+3) +i}=0$ for $j>m$. Then by the identity (\ref{splineiden}), we have
\begin{eqnarray*}
c^{(J+1)} \cdot h^{(J+1)}(x) &=& N\sum_{j=1}^m \sum_{i=1}^{2N+3} \left(c^{(J+1)}\right)_{(j-1)(2N+3) +i} \sigma \left( \langle y_j, x \rangle - t_i\right) \\
&=& \sum_{j=1}^m L_{\textbf t}\left(g_j\right)\left(\langle y_j, x \rangle\right).
\end{eqnarray*}
Combining this with the additive ridge form (\ref{additiveridge}) of $f$ and Lemma \ref{spline}, we know that for $x\in\sph$,
\begin{align*}
&\left|f(x) - c^{(J+1)}\cdot h^{(J+1)}(x) \right|
=\left|\sum_{j=1}^m g_j (\la y_j, x\ra)-\sum_{j=1}^m L_{\textbf t}\left(g_j\right)\left(\langle y_j, x \rangle\right)\right| \\
\leq & \sum_{j=1}^m \|g_j -L_{\mathbf t}(g_j)\|_{C[-1,1]} \leq \sum_{j=1}^m |g_j|_{W^\alpha_\infty} N^{-\alpha}.
\end{align*}
Then the desired error bound is verified.

The total number of free parameters $\mathcal{N}$ in the network is the sum of $J(S+1)$ contributed by $\textbf{w}$,
$J(2S+1)$ by the bias vectors in the first $J$ layers, $1$ by the parameter $B^{(J)}$ in the fully-connected layer,
and $2N+1$ by the vector $\left\{g_j (t_i)\right\}_{i=2}^{2N+2}$ in choosing the coefficient vector $c^{(J+1)}$. So it can be bounded as
$$\mathcal{N} \leq J(S+1)+J(2S+1)+1+ m(2N+1) \leq J (3S+2)+ m(2N+2). $$
This proves \thmref{thm2}.
\end{proof}

\section{Conclusion and Discussion}

In this paper spherical harmonic analysis is conducted rigorously for the approximation theory of deep CNNs
followed by downsampling and one fully connected layer or two on spheres. Our analysis provides rates of uniformly approximating functions
$f \in W^r_\infty (\sph)$ with $r>0$ by deep CNNs followed by two fully connected layers.
To approximate a Lipschitz function in a special additive ridge form, a network with one fully connected layer can be as fast as
one for approximating a univariate Lipschitz function, which demonstrates the super power of deep CNNs in approximating or representing
functions with special structures. Our spherical analysis relies on a special property of the reproducing kernel of $\CH_n^d$ on the sphere.
It would be interesting to extend our technique to approximation of non-smooth functions on $[-1, 1]^d$ and to $L_p$ approximation with $1\leq p <\infty$.

\section*{Acknowledgments}

The authors would like to thank the referees for their
encouraging comments and constructive suggestions that led to an improved presentation of this
paper. The second author is supported partially by the Research Grants Council of Hong Kong [Project \# CityU 21207019] and by City University of Hong Kong [Project \# CityU 7200608].
The last author is supported partially by the Research Grants Council of Hong Kong [Project \# CityU 11306318].

%\bibliography{thesisbib_2}
%\bibliographystyle{plain}

\bibliographystyle{abbrvnat}

\end{document}